\documentclass[envcountresetchap]{svmult}

\usepackage{verbatim}
\usepackage{multirow}
\usepackage{amssymb,amsmath,amsfonts}
\usepackage{subfig}
\usepackage{graphicx,ctable,booktabs}
\usepackage{breqn}
\usepackage{enumerate}
\usepackage{array}
\usepackage{multicol}
\usepackage{algorithm}
\usepackage[noend]{algorithmic}
\usepackage{latexsym}
\usepackage{xspace}
\usepackage{wrapfig}

\newcommand{\algup}{\textnormal{UPUMP}\xspace}

\newcommand{\algkp}{\textnormal{KPUMP}\xspace}

\def\I{\mathcal{I}}
\def\E{\mathcal{E}}

\definecolor{dandelion}{RGB}{240, 225, 48}

\newcommand{\mrmp}{multi-robot motion planning\xspace}

\newcommand{\rsecafter}{\vspace{-0pt}}

\newcommand{\rsubafter}{\vspace{-0pt}}
\newcommand{\rsubbefore}{\vspace{-0pt}}

\title*{$k$-Color Multi-Robot Motion Planning}

\author{Kiril Solovey \and Dan Halperin}

\institute{
    School of Computer Science, Tel-Aviv University,
    \texttt{\{kirilsol,danha\}@post.tau.ac.il}
}

\titlerunning{$k$-Color Multi-Robot Motion Planning}

\date{\today}

\begin{document}
\maketitle

\renewcommand{\abstractname}{Abstract:} 

\section{Introduction}\rsecafter
    \emph{Motion planning} is a fundamental problem in robotics and has applications in different fields such as the study of protein folding, computer graphics, computer-aided design and manufacturing (CAD/CAM), and computer games. The problem of motion planning, in its most basic form, is to find a collision-free path for a robot from start to goal placements while moving in an environment cluttered with obstacles. When several robots share a workspace, a collision between pairs of robots should be taken into consideration as well. This problem is called  \emph{multi-robot motion planning}.

    The first efforts in motion planning in general, and the multi-robot case in particular, were aimed towards the design of \emph{complete} algorithms, guaranteed to find a solution when one exists or report that none exists otherwise. Schwartz and Sharir were the first to design \cite{ss-pm3} a complete algorithm for a multi-robot problem, specifically dealing with the case of coordinating disc robots in the plane.  A work by Hopcroft et al. \cite{hss-cmpmio} presented soon after suggested that, in some cases, the exponential running time may be unavoidable, showing that even the relatively simple setting of the problem is PSPACE-hard in the number of robots. The hardness of the multi-robot problem involving a large number of robots can be attributed to its high number of \emph{degrees of freedom} (or \emph{dofs})---the sum of the dofs of the individual robots. Some efforts were made in the direction of reducing the effective number of dofs. Aronov et al. \cite{avbsv-mpfmr} showed that several dofs can eliminated by forcing the robots to move while maintaining contact, without sacrificing the existence of the solution. van den Berg et al. \cite{bslm-cppmr} proposed a general scheme for decomposing a multi-robot problem into a sequence of subproblems with a reduced number of dofs.

    In recent years, the \emph{sampling-based} approach~\cite{kslo-prm,l-rert} for motion-planning problems has become increasingly popular.
    Unlike the complete planners that explicitly build the \emph{configuration space} of a given problem, the state of all possible configurations of a robot, sampling-based algorithms construct a roadmap that attempts to capture the connectivity of configuration space by sampling this space for valid robot placements and connecting nearby samples. Due to the fact that such techniques perform well in practice, and are simple to implement, they have attracted an immense attention in the last 15 years and were soon applied for solving the multi-robot motion planning problem~\cite{sl-upp} and several tailor-made algorithms for the multi-robot case have been proposed~\cite{hh-hmp,so-cppmr}.

    An abstract form of the multi-robot motion planning problem is the \emph{pebble motion on graphs problem}~\cite{kms-cpmg}. This is a general case of the famous \emph{15-puzzle} where pebbles occupying distinct vertices of a given graph are moved from one set of vertices to another, where the pebbles are bound to move on the edges of the graph. 
    This area of research has been extensively researched and various algorithms and theoretical results exist, see, e.g.,~\cite{cdp-rgg,lb-eccmr,wc-cmpp11,yl-mapp12}.
    [MENTION BEKRIS' WORK]
    
    Even though a fundamental difference between the pebble problem and the multi-robot problem exists---the former is held in a discrete domain while the latter in the continuous---still the two problems are interweaved and pose several points of similarity. Kornhauser et al. were the first~\cite{kms-cpmg} to suggest that an algorithm for the pebble may be found useful in tackling the multi-robot motion planning problem, contemplating ``... it is hoped that the techniques introduces for the solution of the pebble coordination problem may be applicable to special cases of the general geometric problem''.
    
    Recently, the connection between the two problems has earned some interest~\cite{ksb-tdmp12} and was put into practice in the work of Wagner et al.~\cite{wmc-ppp12} that combines their technique for multi-agent pathfinding in discrete environments~\cite{wc-cmpp11} with an implicit representation of the roadmap, presented by \v{S}vestka and Overmars~\cite{so-cppmr}. 
    
    In this context, we developed an algorithm 

    [WRITE HERE ABOUT BEKRIS IN THE CONTEXT OF THE PEBBLE PROBLEM]

    [CONNECTION OF THE PEBBLE PROBLEM WITH THE MULTI-ROBOT PROBLEM]
    * CITE KORNHAUSER
    * CITE M*
    * WRITE ABOUT KPUMP

    [MUTUAL WORK WITH BEKRIS ON MULTI-ROBOT MOTION PLANNING]

\bibliographystyle{spmpsci}
\bibliography{kcolor2}

\end{document}